\documentclass[11pt,a4paper]{article}

\usepackage[hyperref]{acl2019}
\usepackage{times}
\usepackage{latexsym}
\usepackage{microtype}
\usepackage{url}
\aclfinalcopy 

\usepackage{xcolor}
\usepackage{soul}
\usepackage{longtable}
\usepackage{multirow}
\usepackage{booktabs}
\usepackage{amsmath}
\usepackage{amsfonts}
\usepackage{amsthm}
\usepackage{bm}
\usepackage{adjustbox}
\usepackage[utf8x]{inputenc}
\usepackage{tikz}
\usepackage{microtype}
\usepackage{enumitem}
\usepackage{adjustbox}
\usepackage{multicol}
\usepackage{float}
\usepackage[safe]{tipa}

\usetikzlibrary{bayesnet}

\newtheorem{prop}{Proposition}

\usepackage{todonotes}
\makeatletter
\newcommand*\iftodonotes{\if@todonotes@disabled\expandafter\@secondoftwo\else\expandafter\@firstoftwo\fi}  
\makeatother

\newcommand{\cutforCR}[1]{}

\usepackage{cleveref}
\crefname{section}{\S}{\S\S}
\Crefname{section}{\S}{\S\S}
\crefname{table}{Tab.}{}
\crefname{figure}{Fig.}{}
\crefname{algorithm}{Algorithm}{}
\crefname{equation}{eq.}{}
\crefname{appendix}{App.}{}
\crefname{prop}{Proposition}{}
\crefformat{section}{\S#2#1#3}  

\newcommand{\vf}{\bm{f}}
\newcommand{\vg}{\bm{g}}

\newcommand{\veta}{{\boldsymbol \eta}}

\newcommand{\vm}{\bm{m}}
\newcommand{\word}[1]{\textit{#1}}

\newcommand{\calG}{\mathcal{G}}
\newcommand{\calV}{\mathcal{V}}

\newcommand{\calS}{{\cal S}}
\newcommand{\Rpost}{R_\textit{post}}

\title{Unsupervised Discovery of Gendered Language \\through Latent-Variable Modeling}

\author
  {
	\begin{tabular}{lllll}
	Alexander Hoyle\raise1.0ex\hbox{\normalfont\normalsize \textschwa}\raise1.0ex\hbox{\normalfont \normalsize} & Lawrence Wolf-Sonkin\raise1.0ex\hbox{\normalfont\normalsize \textipa{S}}\raise1.0ex\hbox{\normalfont\normalsize} & Hanna Wallach\raise1.0ex\hbox{\normalfont\normalsize \textipa{Z}}
	\end{tabular} \\
		\begin{tabular}{lllll}
	\textbf{Isabelle Augenstein}\raise1.0ex\hbox{\normalfont\normalsize \textipa{P}} & \textbf{Ryan Cotterell}\raise1.0ex\hbox{\normalfont\normalsize \textipa{H}}
	\end{tabular}
	\\
    \raise1.0ex\hbox{\normalfont\normalsize \textschwa}Department of Computer Science, University of Maryland, College Park, USA \thanks{\quad Work undertaken while at University College London Machine Reading group.} \\
    \raise1.0ex\hbox{\normalfont\normalsize \textipa{S}}Department of Computer Science, Johns Hopkins University, Baltimore, USA \\
    \raise1.0ex\hbox{\normalfont\normalsize \textipa{Z}}Microsoft Research, New York City, USA \\
    \raise1.0ex\hbox{\normalfont\normalsize \textipa{H}}Department of Computer Science and Technology, University of Cambridge, Cambridge, UK\\
     \raise1.0ex\hbox{\normalfont\normalsize \textipa{P}}Department of Computer Science, University of Copenhagen, Copenhagen, Denmark\\
	{\tt {hoyle@umd.edu, lawrencews@jhu.edu}} \\
	{\tt {hanna@dirichlet.net,  augenstein@di.ku.dk, rdc42@cam.ac.uk}}
}

\date{}

\begin{document}
\maketitle
\begin{abstract}
Studying the ways in which language is gendered has long been an area
of interest in sociolinguistics. Studies have explored, for example,
the speech of male and female characters in film and the language used
to describe male and female politicians. In this paper, we aim not to
merely study this phenomenon qualitatively, but instead to quantify
the degree to which the language used to describe men and women is
different and, moreover, different in a positive or negative way. To
that end, we introduce a generative latent-variable model that jointly
represents adjective (or verb) choice, with its sentiment, given the
natural gender of a head (or dependent) noun. We find that there are
significant differences between descriptions of male and female nouns
and that these differences align with common gender stereotypes:
Positive adjectives used to describe women are more often related to
their bodies than adjectives used to describe men.
\end{abstract}


\section{Introduction}\label{sec:introduction}

\begin{figure}[ht!]
\centering
   \includegraphics[width=\columnwidth]{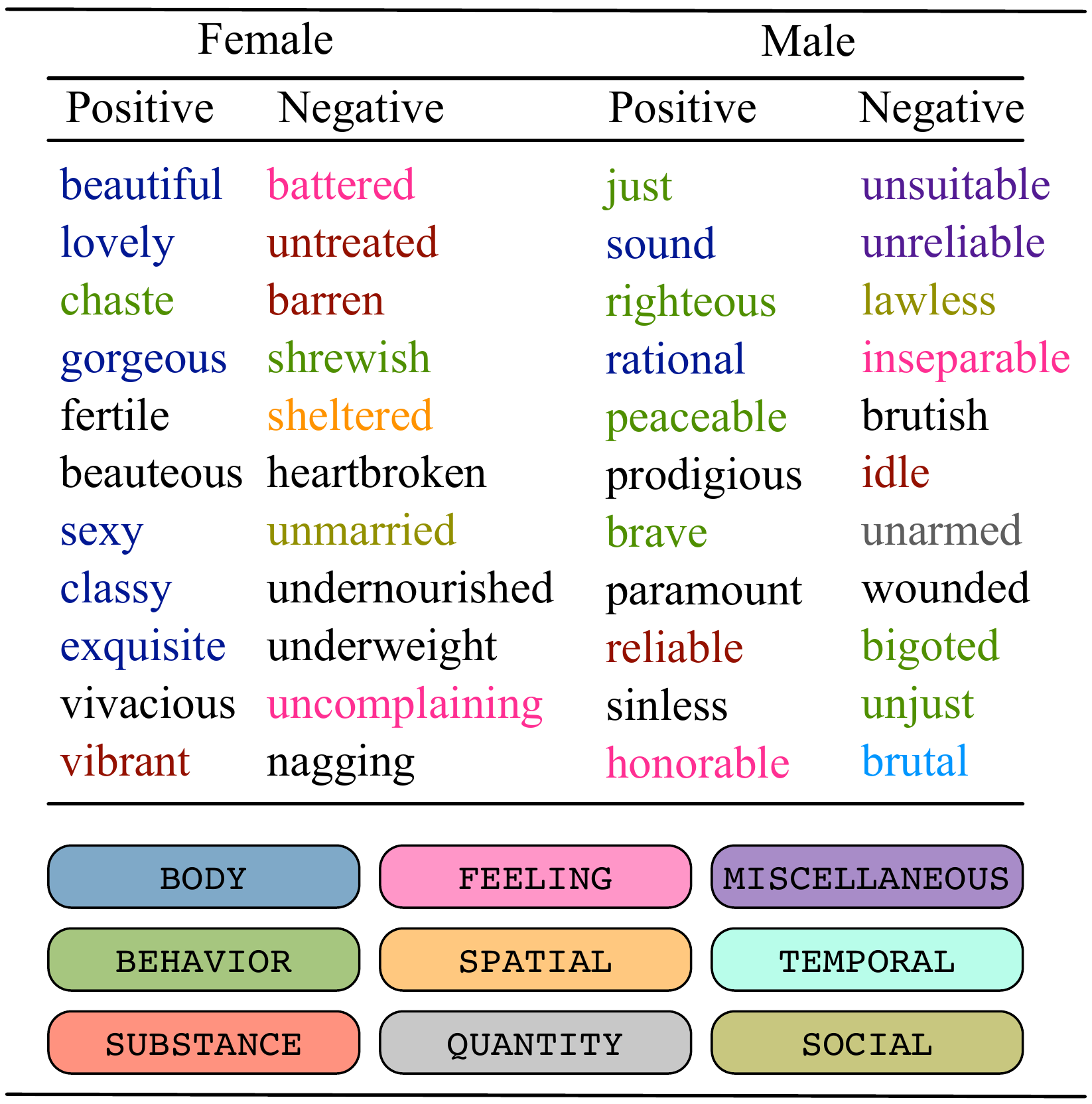}
   \caption{Adjectives, with sentiment, used to describe men and
     women, as represented by our model. Colors indicate the most
     common sense of each adjective from \newcite{TSVETKOV14.1096};
     black indicates out of lexicon. Two patterns are immediately
     apparent: positive adjectives describing women are often related
     to their bodies, while positive adjectives describing men are
     often related to their behavior. These patterns hold generally
     and the differences are significant (see
     \cref{sec:experiments}).\looseness=-1}\label{table:examples}
\end{figure}


Word choice is strongly influenced by gender---both that of the
speaker and that of the referent \cite{lakoff1973language}. Even
within 24 hours of birth, parents describe their daughters as
\textit{beautiful}, \textit{pretty}, and \textit{cute} far more often
than their sons \cite{rubin1974eye}. To date, much of the research in
sociolinguistics on gendered language has focused on laboratory
studies and smaller corpora
\cite{mckee1957differential,williams1975definition,baker2006public};
however, more recent work has begun to focus on larger-scale datasets
\cite{pearceInvestigatingCollocationalBehaviour2008,caldas-coulthardCurvyHunkyKinky2010b,baker2014using,norbergNaughtyBoysSexy2016a}. These
studies compare the adjectives (or verbs) that modify each noun in a
particular gendered pair of nouns, such as \word{boy}--\word{girl},
aggregated across a given corpus. We extend this line of work by
instead focusing on multiple noun pairs simultaneously, modeling how
the choice of adjective (or verb) depends on the natural
gender\footnote{A noun's natural gender is the implied gender of its
  referent (e.g., \textit{actress} refers to woman). We distinguish
  natural gender from grammatical gender because the latter does not
  necessarily convey anything meaningful about the referent.} of the
head (or dependent) noun, abstracting away the noun form. To that end,
we introduce a generative latent-variable model for representing
gendered language, along with sentiment, from a parsed corpus. This
model allows us to quantify differences between the language used to
describe men and women.\looseness=-1

The motivation behind our approach is straightforward: Consider the
sets of adjectives (or verbs) that attach to gendered, animate nouns,
such as \word{man} or \word{woman}. Do these sets differ in ways that
depend on gender? For example, we might expect that the adjective
\word{Baltimorean} attaches to \word{man} roughly the same number of
times as it attaches to \word{woman}, controlling for the frequency of
\word{man} and \word{woman}.\footnote{Men are written about more often
  than women. Indeed, the corpus we use exhibits this trend, as shown
  in \cref{tab:counts}.} But this is not the case for all
adjectives. The adjective \word{pregnant}, for example, almost always
describes women, modulo the rare times that men are described as being
pregnant with, say, emotion. Arguably, the gendered use of
\word{pregnant} is benign---it is not due to cultural bias that women
are more often described as pregnant, but rather because women bear
children. However, differences in the use of other adjectives (or
verbs) may be more pernicious. For example, female professors are less
often described as \word{brilliant} than male professors
\cite{storage2016frequency}, likely reflecting implicit or explicit
stereotypes about men and women.

In this paper, we therefore aim to quantify the degree to which the
language used to describe men and women is different and, moreover,
different in a positive or negative way. Concretely, we focus on three
sociolinguistic research questions about the influence of gender on
adjective and verb choice:

\begin{itemize}

\item[Q1]{What are the \emph{qualitative} differences between the
  language used to describe men and women? For example, what, if any,
  are the patterns revealed by our model? Does the output from our
  model correlate with previous human judgments of gender
  stereotypes?}

\item[Q2]{What are the \emph{quantitative} differences between the
  language used to describe men and women? For example, are adjectives
  used to describe women more often related to their bodies than
  adjectives used to describe men? Can we quantify such patterns using
  existing semantic resources \cite{TSVETKOV14.1096}?}

\item[Q3]{Does the overall \emph{sentiment} of the language used to
  describe men and women differ?}

\end{itemize}

To answer these questions, we introduce a generative latent-variable
model that jointly represents adjective (or verb) choice, with its
sentiment, given the natural gender of a head (or dependent) noun. We
use a form of posterior regularization to guide inference of the
latent variables \cite{ganchev2010posterior}. We then use this model
to study the syntactic $n$-gram corpus of
\cite{goldberg-orwant:2013:*SEM}.\looseness=-1

To answer Q1, we conduct an analysis that reveals differences between
descriptions of male and female nouns that align with common gender
stereotypes captured by previous human judgements. When using our
model to answer Q2, we find that adjectives used to describe women are
more often related to their bodies (significant under a permutation
test with $p < 0.03$) than adjectives used to describe men (see
\cref{table:examples} for examples). This finding accords with
previous research \cite{norbergNaughtyBoysSexy2016a}. Finally, in
answer to Q3, we find no significant difference in the overall
sentiment of the language used to describe men and women.

\begin{table}
\fontsize{10}{10}\selectfont
\centering
  \begin{tabular}{lrlr}
    \toprule
    \multicolumn{2}{c}{Female} &  \multicolumn{2}{c}{Male}\\
    \midrule
    \textit{other} &  2.2 & \textit{other} &   6.8 \\
    daughter       &  1.4 &    husband     &   1.8 \\
    lady           &  2.4 &       king     &   2.1 \\
    wife           &  3.3 &        son     &   2.9 \\
    mother         &  4.2 &     father     &   4.2 \\
    girl           &  5.1 &        boy     &   5.1 \\
    woman          & 11.5 &        man     &  39.9 \\
    \midrule
    Total          & 30.2 &                &  62.7 \\
    \bottomrule
  \end{tabular}
  \caption{Counts, in millions, of male and female nouns present in
    the corpus of
    \newcite{goldberg-orwant:2013:*SEM}.\looseness=-1}\label{tab:counts}
\end{table}

\section{What Makes this Study Different?}\label{sec:data}

As explained in the previous section, many sociolinguistics
researchers have undertaken corpus-based studies of gendered
language. In this section, we therefore differentiate our approach
from these studies and from recent NLP research on gender biases in
word embeddings and co-reference systems.\looseness=-1

\paragraph{Syntactic collocations and noun types.}
Following the methodology employed in previous sociolinguistic studies
of gendered language, we use syntactic collocations to make definitive
claims about gendered relationships between words. This approach
stands in contrast to bag-of-words analyses, where information about
gendered relationships must be indirectly inferred. By studying the
adjectives and verbs that attach to gendered, animate nouns, we are
able to more precisely quantify the degree to which the language used
to describe men and women is different. To date, much of the
corpus-based sociolinguistics research on gendered language has
focused on differences between the adjectives (or verbs) that modify
each noun in a particular gendered pair of nouns, such as
\word{boy}--\word{girl} or \word{man}--\word{woman} (e.g.,
\newcite{pearceInvestigatingCollocationalBehaviour2008,caldas-coulthardCurvyHunkyKinky2010b,norbergNaughtyBoysSexy2016a}). To
assess the differences, researchers typically report top
collocates\footnote{Typically ranked by the log of the Dice
  coefficient.} for one word in the pair,
exclusive of collocates for the other. This approach has the effect of
restricting both the amount of available data and the claims that can
be made regarding gendered nouns more broadly. In contrast, we focus
on multiple noun pairs (including plural forms) simultaneously,
modeling how the choice of adjective (or verb) depends on the natural
gender of the head (or dependent) noun, abstracting away the noun
form. As a result, we are able to make broader claims.

\paragraph{The corpus of \newcite{goldberg-orwant:2013:*SEM}.}
To extract the adjectives and verbs that attach to gendered, animate
nouns, we use the corpus of \newcite{goldberg-orwant:2013:*SEM}, who
ran a then-state-of-the-art dependency parser on 3.5 million
digitalized books. We believe that the size of this corpus (11 billion
words) makes our study the largest collocational study of its
kind. Previous studies have used corpora of under one billion words,
such as the British National Corpus (100 million words)
\cite{pearceInvestigatingCollocationalBehaviour2008}, the New Model
Corpus (100 million words) \cite{norbergNaughtyBoysSexy2016a}, and the
Bank of English Corpus (450 million words)
\cite{moonrosamundGorgeousGrumpyAdjectives2014}. By default, the
corpus of \newcite{goldberg-orwant:2013:*SEM} is broken down by year,
but we aggregate the data across years to obtain roughly 37 million
noun--adjectives pairs, 41 million \textsc{nsubj}--verb pairs, and 14
million \textsc{dobj}--verb pairs. We additionally lemmatize each
word. For example, the noun \word{stewardesses} is lemmatized to a set
of lexical features consisting of the genderless lemma
\textsc{steward} and the morphological features $+\textsc{fem}$ and
$+\textsc{pl}$. This parsing and lemmatization process is illustrated
in \cref{fig:dependency-example}.\looseness=-1

\begin{figure}
\centering
    \includegraphics[width=\columnwidth]{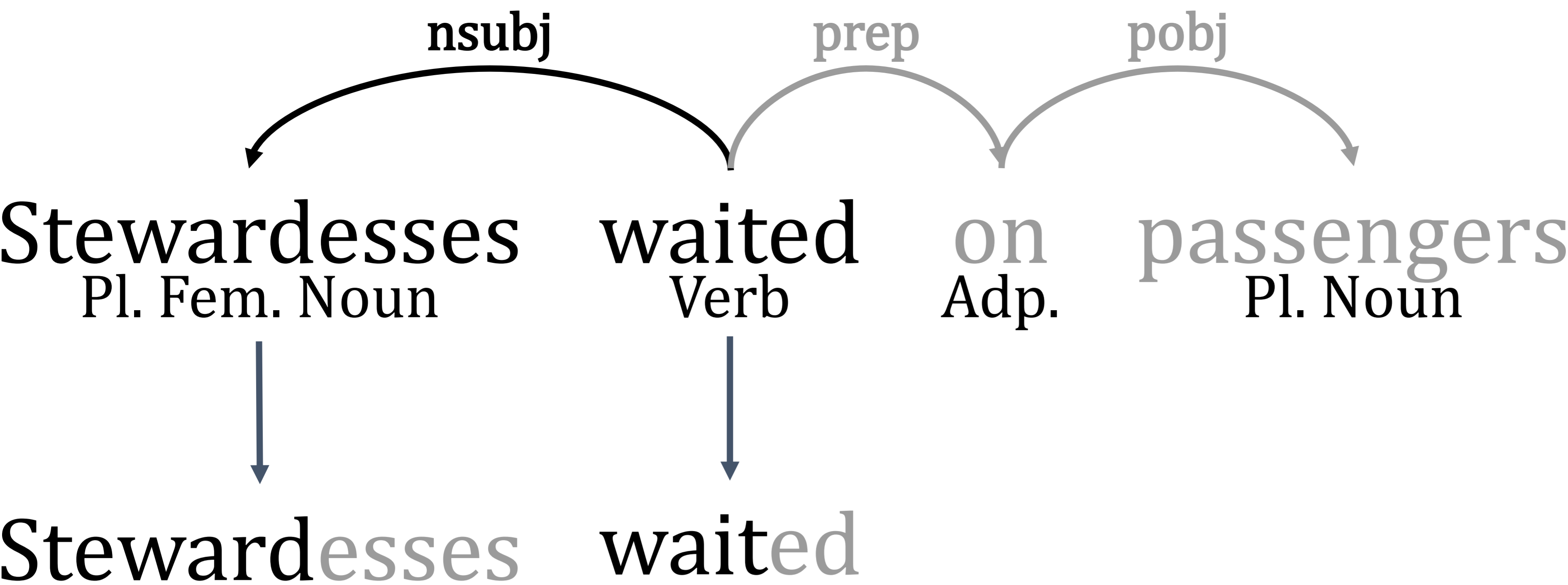}
    \caption{An example sentence with its labeled dependency parse (top) and lemmatized words (bottom).}
    \label{fig:dependency-example}
\end{figure}

\paragraph{Quantitative evaluation.}
Our study is also quantitative in nature: we test concrete hypotheses
about differences between the language used to describe men and
women. For example, we test whether women are more often described
using adjectives related to their bodies and emotions. This
quantitative focus differentiates our approach from previous
corpus-based sociolinguistics research on gendered language. Indeed,
in the introduction to a special issue on corpus methods in the
journal \word{Gender and Language}, \newcite{baker2013introduction}
writes, ``while the term corpus and its plural corpora are reasonably
popular within Gender and Language (occurring in almost 40\% of
articles from issues 1-6), authors have mainly used the term as a
synonym for `data set' and have tended to carry out their analysis by
hand and eye methods alone.'' Moreover, in a related paper on
extracting gendered language from word embeddings,
\newcite{garg2018word} lament that ``due to the relative lack of
systematic quantification of stereotypes in the literature [... they]
cannot directly validate [their] results.'' For an overview of
quantitative evaluation, we recommend
\newcite{baker2014using}.\looseness=-1

\paragraph{Speaker versus referent.}
Many data-driven studies of gender and language focus on what speakers
of different genders say rather than differences between descriptions
of men and women. This is an easier task---the only annotation
required is the gender of the speaker. For example,
\newcite{ott2016tweet} used a topic model to study how word choice in
tweets is influenced by the gender of the tweeter;
\newcite{schofield-mehr:2016:CLfL2016} modeled gender in film dialog;
and, in the realm of social media analysis, \newcite{Bamman2014}
discussed stylistic choices that enable classifiers to distinguish
between tweets written by men versus women.\looseness=-1

\paragraph{Model versus data.}
Recent NLP research has focused on gender biases in word embeddings
\cite{bolukbasi2016man,zhao-EtAl:2017:EMNLP20173} and co-reference
systems \cite{N18-2003,N18-1067}. These papers are primarily concerned
with mitigating biases present in the output of machine learning
models deployed in the real world \cite{o2016weapons}. For example,
\newcite{bolukbasi2016man} used pairs of gendered words, such as
\word{she}--\word{he}, to mitigate unwanted gender biases in word
embeddings. Although it is possible to rank the adjectives (or verbs)
most aligned with the embedding subspace defined by a pair of gendered
words, there are no guarantees that the resulting adjectives (or
verbs) were specifically used to describe men or women in the dataset
from which the embeddings were learned. In contrast, we use syntactic
collocations to explicitly represent gendered relationships between
individual words. As a result, we are able make definitive claims
about these relationships, thereby enabling us to answer
sociolinguistic research questions. Indeed, it is this sociolinguistic
focus that differentiates our approach from this line of work.

\section{Modeling Gendered Language}\label{sec:model}

As explained in \cref{sec:introduction}, our aim is quantify the
degree to which the language used to describe men and women is
different and, moreover, different in a positive or negative way. To
do this, we therefore introduce a generative latent-variable model
that jointly represents adjective (or verb) choice, with its
sentiment, given the natural gender of a head (or dependent)
noun. This model, which is based on the sparse additive generative
model
\cite[SAGE;][]{eisensteinSparseAdditiveGenerative2011},\footnote{SAGE
  is a flexible alternative to latent Dirichlet allocation
  \cite[LDA;][]{blei2003latent}---the most widely used statistical
  topic model. Our study could also have been conducted using LDA;
  drawing on SAGE was primarily a matter of personal taste.} enables
us to extract ranked lists of adjectives (or verbs) that are used,
with particular sentiments, to describe male or female nouns.

We define $\mathcal{G}$ to be the set of gendered, animate nouns in
our corpus and $n \in \mathcal{G}$ to be one such noun. We represent
$n$ via a multi-hot vector $\vf_n \in \{0,1\}^T$ of its lexical
features---i.e., its genderless lemma, its gender (male or female),
and its number (singular or plural). In other words, $\vf_n$ always
has exactly three non-zero entries; for example, the only non-zero
entries of $\vf_{\word{stewardesses}}$ are those corresponding to
$\textsc{steward}$, $+\textsc{fem}$, and $+\textsc{pl}$. We define
$\mathcal{V}$ to be the set of adjectives (or verbs) in our corpus and
$\nu \in \mathcal{V}$ to be one such adjective (or verb). To simplify
exposition, we refer to each adjective (or verb) that attaches to noun
$n$ as a \emph{neighbor} of $n$. Finally, we define $\mathcal{S} =
\{\textsc{pos}, \textsc{neg}, \textsc{neu}\}$ to be a set of three
sentiments and $s \in \mathcal{S}$ to be one such
sentiment.\looseness=-1

\begin{figure}
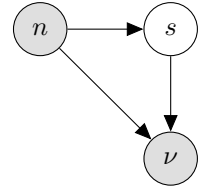

    \centering
    \tikz{ %
        \node[obs] (N) {$n$} ;
        \node[latent,  right=of N] (S) {$s$} ;
        \node[obs, below=of S] (R) {$\nu$} ;
        \edge {N} {S};
        \edge {N} {R};
        \edge {S} {R};
    }
    \caption{Graphical model depicting our model's representation of
      nouns, neighbors, and (latent) sentiments.}
    \label{fig:joint_sentiment}
\end{figure}

Drawing inspiration from SAGE, our model jointly represents nouns,
neighbors, and (latent) sentiments as depicted in
\cref{fig:joint_sentiment}. Specifically,
\begin{equation}
  \label{eqn:joint}
  p(\nu, n, s) = p(\nu \,|\, s, n)\,p(s \,|\, n)\,p(n).
\end{equation}
The first factor in \cref{eqn:joint} is defined as
\begin{equation}
p(\nu \,|\, s, n) \propto \exp\{m_\nu + \vf_n^{\top} \veta(\nu, s)\},
\end{equation}
where $\vm \in \mathbb{R}^{|\mathcal{V}|}$ is a background
distribution and $\veta(\nu, s) \in \mathbb{R}^T$ is a neighbor- and
sentiment-specific deviation. The second factor in \cref{eqn:joint} is
defined as\looseness=-1
\begin{equation}
  p(s \,|\, n) \propto \exp{(\omega_s^n)},
\end{equation}
where $\omega_s^n \in \mathbb{R}$, while the third factor is defined as
\begin{equation}
  p(n) \propto \exp{(\xi_n)},
\end{equation}
where $\xi_n \in \mathbb{R}$. We can then extract lists of neighbors that are used, with particular sentiments, to describe male and female nouns, ranked by scores that are a function of their deviations. For example,  the score for neighbor $\nu$ when used, with positive sentiment, to describe a male noun is defined as
\begin{equation}
  \label{eqn:topic}
  \tau_\textsc{masc-pos}(\nu) \propto \exp\{\vg_\textsc{masc}^{\top} \veta(\nu, \textsc{pos})\},
\end{equation}
where $\vg_\textsc{masc} \in \{0, 1\}^T$ is a vector where only the
entry that corresponds to $+\textsc{masc}$ is non-zero.\looseness=-1

Because our corpus does not contain explicit sentiment information, we
marginalize out $s$:
\begin{equation}
p(\nu, n) = \sum_{s \in \mathcal{S}} p(\nu \,|\, s, n)\, p(s \,|\, n)\, p(n).
\end{equation}
This yields the following objective function:
\begin{equation}
  \label{eq:obj}
  \sum_{n \in \mathcal{G}} \sum_{\nu \in \mathcal{V}} \hat{p}(\nu, n) \log{(p(\nu, n))},
\end{equation}
where $\hat{p}(\nu, n) \propto \#(\nu, n)$ is the empirical
probability of neighbor $\nu$ and noun $n$ in our corpus.

To ensure that the latent variables in our model correspond to
positive, negative, and neutral sentiments, we rely on posterior
regularization \cite{ganchev2010posterior}. Given an additional
distribution $q(s \,|\, \nu)$ that provides external information about
the sentiment of neighbor $\nu$, we regularize $p(s \,|\, \nu)$, as
defined by our model, to be close (in the sense of KL-divergence) to
$q(s \,|\, \nu)$. Specifically, we construct the following posterior
regularizer:
\begin{align}\label{eq:reg}
  &\Rpost \notag\\
  &\quad =\text{KL}( q(s \,|\, \nu) \,||\, p(s \,|\, \nu)) \\
  &\quad = -\sum_{s \in \calS} q (s \,|\, \nu) \log{( p(s \,|\, \nu))} + H(q),
\end{align}
where $H(q)$ is constant and $p(s \,|\, \nu)$ is defined as
\begin{align}
  p(s \,|\, \nu) &= \sum_{n \in \calG} p(s, n \,|\, \nu) \\
  &= \sum_{n \in
    \calG} \frac{p(\nu \,|\, n, s)\, p(s \,|\, n)\,p(n)}{p(\nu)}.
\end{align}
We use the combined sentiment lexicon of \newcite{hoyleSentiment2019}
as $q(s \,|\, \nu)$. This lexicon represents each word's sentiment as
a three-dimensional Dirichlet distribution, thereby accounting for the
relative confidence in the strength of each sentiment and, in turn,
accommodating polysemous and rare words. By using the lexicon as
external information in our posterior regularizer, we can control the
extent to which it influences the latent variables.

We add the regularizer in \cref{eq:reg} to the objective function in
\cref{eq:obj}, using a multiplier $\beta$ to control the strength of
the posterior regularization. We also impose an $L_1$-regularizer
$\alpha \cdot ||\veta||_1$ to induce sparsity. The complete objective
function is then
\begin{align}
 &\sum_{n \in \mathcal{G}} \sum_{\nu\in \mathcal{V}} \hat{p}(\nu, n) \log{(p(\nu, n))}\notag\\
& \quad + \alpha \cdot ||\veta||_1 + \beta \cdot \Rpost.  \label{eq:objective}
\end{align}
We optimize \cref{eq:objective} with respect to $\veta(\cdot, \cdot)$,
${\boldsymbol \omega}$, and ${\boldsymbol \xi}$ using the Adam
optimizer \cite{Kingma2014AdamAM} with $\alpha$ and $\beta$ set as
described in \cref{sec:experiments}. To ensure that the parameters are
interpretable (e.g., to avoid a negative
$\veta(\textsc{pregnant},\textsc{neg})$ canceling out a positive
$\veta(\textsc{pregnant},\textsc{pos})$)), we also constrain
$\veta(\cdot,\cdot)$ to be non-negative, although without this
constraint, our results are largely the same.

\paragraph{Relationship to pointwise mutual information.}
Our model also recovers pointwise mutual information (PMI),
which has been used previously to identify gendered
language~\cite{rudinger2017social}.\looseness=-1
\begin{prop}\label{thm:prop1}
  Consider
the following restricted version of our model.
Let $\vf_{g} \in \{0, 1\}^2$ be a one-hot vector that represents only
the gender of a noun $n$.
We write $g$ instead of $n$, equivalence-classing all nouns as either \textsc{masc} or \textsc{fem}.
Let $\veta^\star(\cdot) : \calV \rightarrow \mathbb{R}^2$
be the maximum-likelihood estimate for the special case of our model
without (latent) sentiments:
\begin{equation}
    p(\nu \mid g) \propto \exp(m_\nu + \vf_g^{\top}\veta^\star(\nu)).
\end{equation}
Then, we have
\begin{align}
   \tau_g(\nu) \propto \exp(\textrm{PMI}(\nu, g)).
\end{align}
\end{prop}
\begin{proof}
See \cref{sec:pmi}.
\end{proof}
\cref{thm:prop1} says that if we use a limited set of lexical features
(i.e., only gender) and estimate our model \emph{without} any
regularization or latent sentiments, then ranking the neighbors by
$\tau_g(\nu)$ (i.e., by their deviations from the background
distribution) is equivalent to ranking them by their PMI. This
proposition therefore provides insight into how our model builds on
PMI. Specifically, in contrast to PMI, 1) our model can consider
lexical features other than gender, 2) our model is regularized to
avoid the pitfalls of maximum-likelihood estimation, and 3) our model
cleanly incorporates latent sentiments, relying on posterior
regularization to ensure that the $p(s \,|\, \nu)$ is close to the
sentiment lexicon of \newcite{hoyleSentiment2019}.


\begin{table*}
  \centering
  \begin{adjustbox}{width=\textwidth}
  \begin{tabular}{lrlrlr|lrlrlr}
  \toprule
  \multicolumn{2}{c}{$\tau_\textsc{masc-pos}$}&  \multicolumn{2}{c}{$\tau_\textsc{masc-neg}$}& \multicolumn{2}{c|}{$\tau_\textsc{masc-neu}$}& \multicolumn{2}{c}{$\tau_\textsc{fem-pos}$}&  \multicolumn{2}{c}{$\tau_\textsc{fem-neg}$}& \multicolumn{2}{c}{$\tau_\textsc{fem-neu}$} \\
  Adj. &  Value & Adj. & Value & Adj. & Value & Adj. &  Value & Adj. & Value & Adj. & Value\\
  \midrule
      faithful &   2.3 &         unjust &   2.4 &          german &   1.9 &        pretty &  3.3 &      horrible &  1.8 &        virgin &  2.8 \\
   responsible &   2.2 &           dumb &   2.3 &        teutonic &   0.8 &          fair &  3.3 &   destructive &  0.8 &       alleged &  2.0 \\
   adventurous &   1.9 &        violent &   1.8 &       financial &   2.6 &     beautiful &  3.4 &     notorious &  2.6 &        maiden &  2.8 \\
         grand &   2.6 &           weak &   2.0 &          feudal &   2.2 &        lovely &  3.4 &        dreary &  0.8 &       russian &  1.9 \\
        worthy &   2.2 &           evil &   1.9 &           later &   1.6 &      charming &  3.1 &          ugly &  3.2 &          fair &  2.6 \\
         brave &   2.1 &         stupid &   1.6 &        austrian &   1.2 &         sweet &  2.7 &         weird &  3.0 &       widowed &  2.4 \\
          good &   2.3 &          petty &   2.4 &       feudatory &   1.8 &         grand &  2.6 &       harried &  2.4 &         grand &  2.1 \\
        normal &   1.9 &         brutal &   2.4 &        maternal &   1.6 &       stately &  3.8 &      diabetic &  1.2 &     byzantine &  2.6 \\
     ambitious &   1.6 &         wicked &   2.1 &        bavarian &   1.5 &    attractive &  3.3 &  discontented &  0.5 &   fashionable &  2.5 \\
       gallant &   2.8 &     rebellious &   2.1 &           negro &   1.5 &        chaste &  3.3 &      infected &  2.8 &          aged &  1.8 \\
        mighty &   2.4 &            bad &   1.9 &        paternal &   1.4 &      virtuous &  2.7 &     unmarried &  2.8 &       topless &  3.9 \\
         loyal &   2.1 &      worthless &   1.6 &        frankish &   1.8 &       fertile &  3.2 &       unequal &  2.4 &      withered &  2.9 \\
       valiant &   2.8 &        hostile &   1.9 &           welsh &   1.7 &    delightful &  2.9 &       widowed &  2.4 &      colonial &  2.8 \\
     courteous &   2.6 &       careless &   1.6 &  ecclesiastical &   1.6 &        gentle &  2.6 &       unhappy &  2.4 &      diabetic &  0.7 \\
      powerful &   2.3 &         unsung &   2.4 &           rural &   1.4 &    privileged &  1.4 &        horrid &  2.2 &     burlesque &  2.9 \\
      rational &   2.1 &        abusive &   1.5 &         persian &   1.4 &      romantic &  3.1 &       pitiful &  0.8 &        blonde &  2.9 \\
       supreme &   1.9 &      financial &   3.6 &          belted &   1.4 &     enchanted &  3.0 &     frightful &  0.5 &      parisian &  2.7 \\
   meritorious &   1.5 &         feudal &   2.5 &           swiss &   1.3 &        kindly &  3.2 &    artificial &  3.2 &          clad &  2.5 \\
        serene &   1.4 &          false &   2.3 &         finnish &   1.1 &       elegant &  2.8 &        sullen &  3.1 &        female &  2.3 \\
       godlike &   2.3 &         feeble &   1.9 &        national &   2.2 &          dear &  2.2 &    hysterical &  2.8 &      oriental &  2.2 \\
         noble &   2.3 &       impotent &   1.7 &        priestly &   1.8 &       devoted &  2.0 &         awful &  2.6 &       ancient &  1.7 \\
      rightful &   1.9 &      dishonest &   1.6 &     merovingian &   1.6 &     beauteous &  3.9 &       haughty &  2.6 &      feminist &  2.9 \\
         eager &   1.9 &     ungrateful &   1.5 &        capetian &   1.4 &     sprightly &  3.2 &      terrible &  2.4 &      matronly &  2.6 \\
     financial &   3.3 &     unfaithful &   2.6 &        prussian &   1.4 &       beloved &  2.5 &        damned &  2.4 &        pretty &  2.5 \\
    chivalrous &   2.6 &    incompetent &   1.7 &          racial &   0.9 &      pleasant &  1.8 &       topless &  3.5 &       asiatic &  2.0 \\
 \bottomrule
  \end{tabular}
  \end{adjustbox}
  \caption{For each sentiment, we provide the largest-deviation
    adjectives used to describe male and female
    nouns.}\label{tab:adj-results}
\end{table*}




\section{Experiments, Results, and Discussion}\label{sec:experiments}

We use our model to study the corpus of
\newcite{goldberg-orwant:2013:*SEM} by running it separately on the
noun--adjectives pairs, the \textsc{nsubj}--verb pairs, and the
\textsc{dobj}--verb pairs. We provide a full list of the lemmatized,
gendered, animate nouns in \cref{app:gendered-nouns}. We use $\alpha
\in \{0, 10^{-5}, 10^{-4}, 0.001, 0.01 \}$ and $\beta \in \{10^{-5},
10^{-4}, 0.001, 0.01, 0.1, 1, 10, 100\}$; when we report results
below, we use parameter values averaged over these hyperparameter
settings.


\subsection{Q1: \emph{Qualitative} Differences}

Our first research question concerns the qualitative differences
between the language used to describe men and women. To answer this
question, we use our model to extract ranked lists of neighbors that
are used, with particular sentiments, to describe male and female
nouns. As explained in \cref{sec:model}, we rank the neighbors by
their deviations from the background distribution (see, for example,
\cref{eqn:topic}).

\paragraph{Qualitative evaluation.}

In \cref{tab:adj-results}, we provide, for each sentiment, the 25
largest-deviation adjectives used to describe male and female
nouns. The results are striking: it is immediately apparent that
positive adjectives describing women are often related to their
appearance (e.g., \word{beautiful}, \word{fair}, and
\word{pretty}). Sociolinguistic studies of other corpora, such as
British newspapers \cite{caldas-coulthardCurvyHunkyKinky2010b}, have
also revealed this pattern. Adjectives relating to fertility, such as
\word{fertile} and \word{barren}, are also more prevalent for
women. We provide similar tables for verbs in \cref{sec:output}.
Negative verbs describing men are often related to violence (e.g.,
\word{murder}, \word{fight}, \word{kill}, and
\word{threaten}). Meanwhile, women are almost always the object of
\word{rape}, which aligns with our knowledge of the world and supports
the collocation of \word{rape} and \word{girl} found by
\newcite{baker2014using}. Broadly speaking, positive verbs describing
men tend to connote virtuosity (e.g., \word{gallant} and
\word{inspire}), while those describing women appear more trivial
(e.g., \word{sprightly}, \word{giggle}, and
\word{kiss}). 

\paragraph{Correlation with human judgments.} 
To determine whether the output from our model accords with previous
human judgements of gender stereotypes, we use the corpus of
\newcite{williams1975definition}, which consists of 63 adjectives
annotated with (binary) gender stereotypes. We measure Spearman's
$\rho$ between these annotations and the probabilities output by our
model. We find a relatively strong positive correlation of $\rho=0.59$
($p < 10^{-6}$), which indicates that the output from our model aligns
with common gender stereotypes captured by previous human
judgements. We also measure the correlation between continuous
annotations of 300 adjectives from two follow-up studies
\cite{williams+best,williamsSexStereotypesTrait1977}\footnote{The
  studies consider the same set of words 20 years apart; we average
  their annotations, obtained from \newcite{garg2018word}.}  and the
probabilities output by our model. Here, the correlation is
$\rho=0.33$ ($p < 10^{-8}$), and the binarized annotations agree with
the output from our model for $64\%$ of terms. We note that some of
the disagreement is due to reporting bias
\cite{gordonReportingBiasKnowledge2013} in our corpus. For example,
only men are described in our corpus as \word{effeminate}, although
humans judge it to be a highly feminine adjective.



\subsection{Q2: \emph{Quantitative} differences}\label{sec:q2}

Our second research question concerns the quantitative differences
between the language used to describe men and women. To answer this
question, we use two existing semantic resources---one for adjectives
\cite{TSVETKOV14.1096} and one for verbs
\cite{miller1993semantic}---to quantify the patterns revealed by our
model. Again, we use our model to extract ranked lists of neighbors
that are used, with particular sentiments, to describe male and female
nouns. We consider only the 200 largest-deviation neighbors for each
sentiment and gender. This restriction allows us to perform an
unpaired permutation test \cite{good2004permutation} to determine
whether there are significant differences between the language used to
describe men and women.\looseness=-1

\paragraph{Adjective evaluation.}

Women are supposedly more often described using adjectives related to
their bodies and emotions. For example, \newcite{de1953second} writes
that ``from girlhood, women are socialized to live and experience
their bodies as objects for another's gaze...'' Although studies of
reasonably large corpora have found evidence to support this
supposition \cite{norbergNaughtyBoysSexy2016a}, none have done so at
scale with statistical significance testing. We use the semantic
resource of \newcite{TSVETKOV14.1096}, which categorizes adjectives
into thirteen senses: \textsc{behavior}, \textsc{body},
\textsc{feeling}, \textsc{mind}, etc. Specifically, each adjective has
a distribution over senses, capturing how often the adjective
corresponds to each sense. We analyze the largest-deviation adjectives
for each sentiment and gender by computing the frequency with which
these adjectives correspond to each sense. We depict these frequencies
in \cref{fig:adjective-results}. Specifically, we provide frequencies
for the senses where, after Bonferroni correction, the differences
between men and women are significant. We find that adjectives used to
describe women are indeed more often related to their bodies and
emotions than adjectives used to describe men.\looseness=-1

\begin{figure}
\centering
\begin{adjustbox}{width=1.0\columnwidth}
\includegraphics{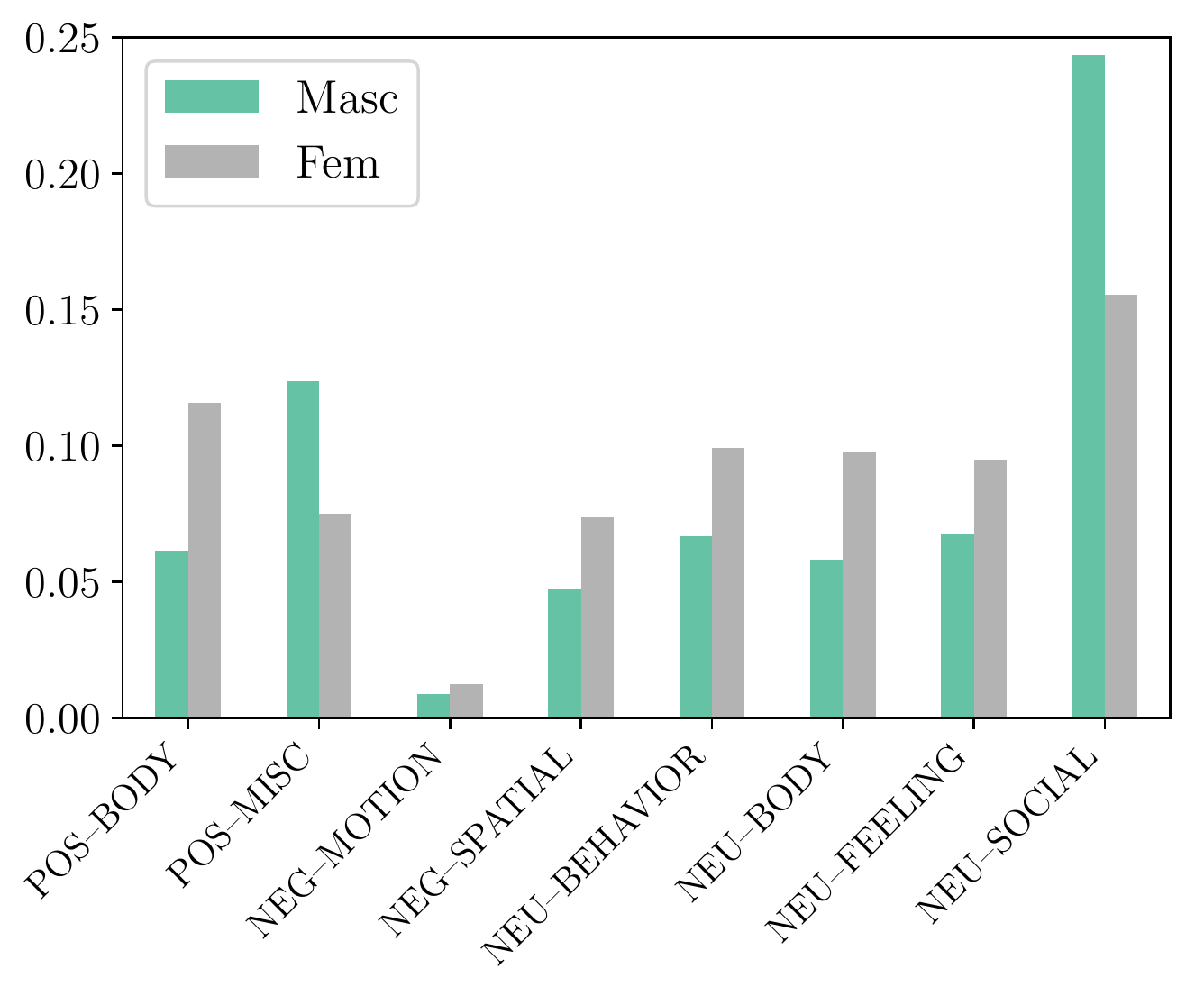}
\end{adjustbox}
\caption{The frequency with which the 200 largest-deviation adjectives
  for each sentiment and gender correspond to each sense from
  \newcite{TSVETKOV14.1096}.}
\label{fig:adjective-results}
\end{figure}

\paragraph{Verb evaluation.}

To evalaute verbs senses, we take the same approach as for
adjectives. We use the semantic resource of
\newcite{miller1993semantic}, which categorizes verbs into fifteen
senses. Each verb has a distribution over senses, capturing how often
the verb corresponds to each sense. We consider two cases: the
\textsc{nsubj}--verb pairs and the \textsc{dobj}--verb pairs. Overall,
there are fewer significant differences for verbs than there are for
adjectives. There are no statistically significant differences for the
\textsc{dobj}--verb pairs. We depict the results for the
\textsc{nsubj}--verb pairs in \cref{fig:verb-results}. We find that
verbs used to describe women are more often related to their bodies
than verbs used to describe men.

\begin{figure}
\centering
\begin{adjustbox}{width=1.0\columnwidth}
\includegraphics{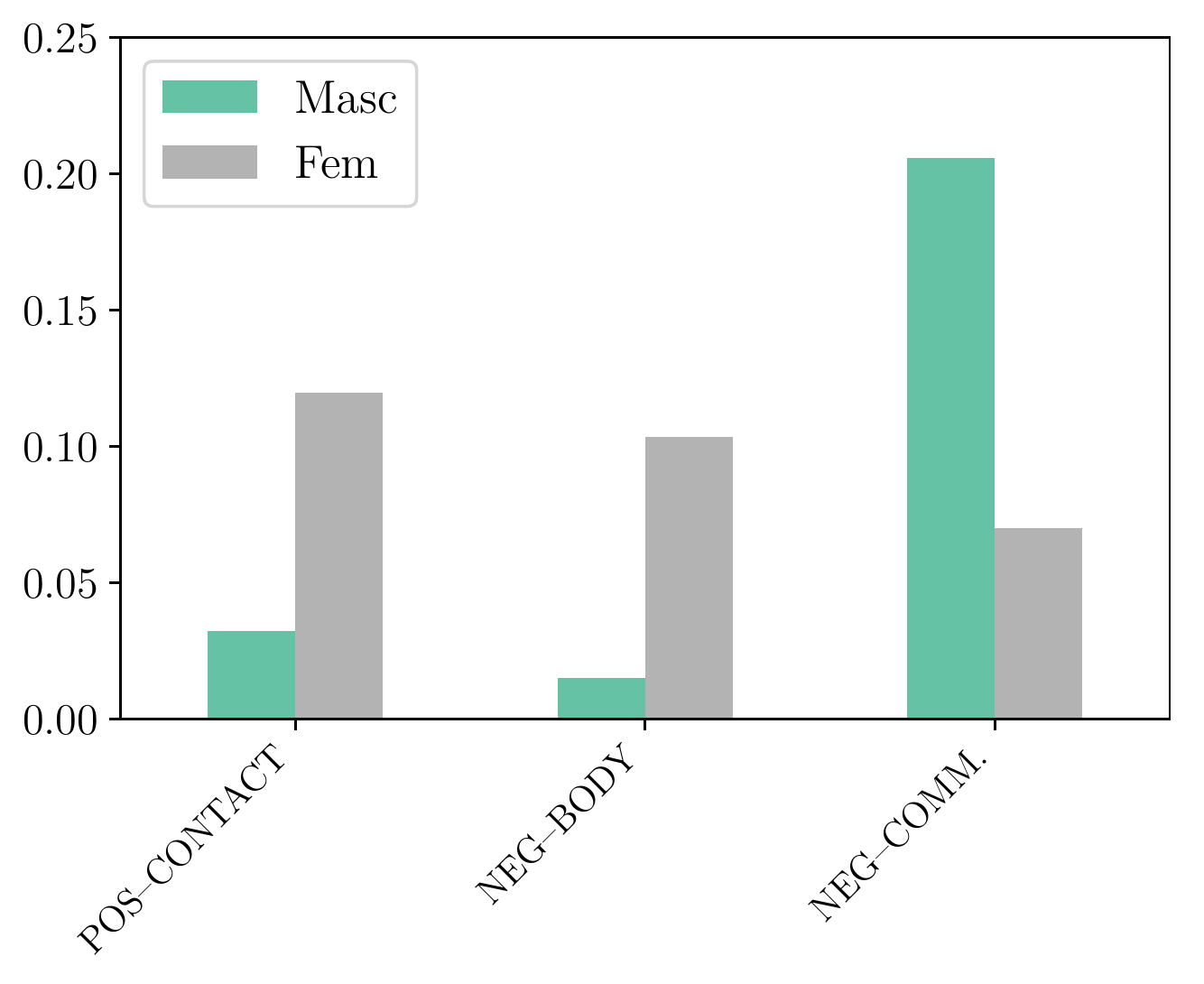}
\end{adjustbox}
\caption{The frequency with which the 200 largest-deviation verbs for
  each sentiment and gender correspond to each sense from
  \newcite{miller1993semantic}. These results are only for the
  \textsc{nsubj}--verb pairs; there are no statistically significant
  differences for \textsc{dobj}--verb pairs.}\label{fig:verb-results}
\end{figure}


\begin{table}
\fontsize{10}{10}\selectfont
    \centering
  \begin{adjustbox}{width=\columnwidth}
	\begin{tabular}{lllllll} \toprule
    	&            \multicolumn{2}{c}{\textsc{adj}} &  \multicolumn{2}{c}{\textsc{nsubj}} &  \multicolumn{2}{c}{\textsc{dobj}}	\\ \cmidrule(l){2-3} \cmidrule(l){4-5} \cmidrule(l){6-7}
        & \textsc{msc} & \textsc{fem}  & \textsc{msc} & \textsc{fem}  & \textsc{msc} & \textsc{fem} \\ \midrule
        \textsc{pos} &  0.34 & 0.38 & 0.37 & 0.36 & 0.37 & 0.36 \\
        \textsc{neg} &  0.30 & 0.31 & 0.33 & 0.34 & 0.34 & 0.35 \\
        \textsc{neu} &  \textbf{0.36} & \textbf{0.31} & 0.30 & 0.30 & 0.30 & 0.29 \\ \bottomrule
    \end{tabular}
    \end{adjustbox}
  \caption{The frequency with which the 200 largest-deviation
    neighbors for each gender correspond to each sentiment, obtained
    using a simplified version of our model and the lexicon of
    \newcite{hoyleSentiment2019}. Significant differences ($p <
    0.05/3$ under an unpaired permutation test with Bonferroni
    correction) are in bold.\looseness=-1}
    \label{tab:sentiment}
\end{table}

\subsection{Q3: Differences in \emph{sentiment}}

Our final research question concerns the overall sentiment of the
language used to describe men and women. To answer this question, we
use a simplified version of our model, without the latent sentiment
variables or the posterior regularizer. We are then able to use the
combined sentiment lexicon of \newcite{hoyleSentiment2019} to analyze
the largest-deviation neighbors for each gender by computing the
frequency with which each neighbor corresponds to each sentiment. We
report these frequencies in \cref{tab:sentiment}. We find that there
is only one significant difference: adjectives used to describe men
are more often neutral than those used to describe women.

\section{Conclusion and Limitations}\label{sec:limitations}

We presented an experimental framework for quantitatively studying the
ways in which the language used to describe men and women is different
and, moreover, different in a positive or negative way. We introduced
a generative latent-variable model that jointly represents adjective
(or verb) choice, with its sentiment, given the natural gender of a
head (or dependent) noun. Via our experiments, we found evidence in
support of common gender stereotypes. For example, positive adjectives
used to describe women are more often related to their bodies than
adjectives used to describe men. Our study has a few limitations that
we wish to highlight. First, we ignore demographics (e.g., age,
gender, location) of the speaker, even though such demographics are
likely influence word choice. Second, we ignore genre (e.g., news,
romance) of the text, even though genre is also likely to influence
the language used to describe men and women. In addition, depictions 
of men and women have certainly changed over the period covered by our corpus; indeed, \newcite{underwood2018transformation} found evidence of such a change for fictional characters.
In future work, we intend to conduct a diachronic analysis in English using the same corpus,
in addition to a cross-linguistic study of gendered language.\looseness=-1





\section*{Acknowledgments}
We would like to thank the three anonymous ACL 2019 reviewers for
their comments on the submitted version, as well as the anonymous
reviewers of a previous submission. We would also like to thank Adina
Williams and Eleanor Chodroff for their comments on versions of the manuscript. The last author would like to acknowledge a Facebook fellowship. 

\bibliography{naaclhlt2019}

\begin{thebibliography}{36}
\expandafter\ifx\csname natexlab\endcsname\relax\def\natexlab#1{#1}\fi

\bibitem[{Baker(2005)}]{baker2006public}
Paul Baker. 2005.
\newblock \emph{Public discourses of gay men}.
\newblock Routledge.

\bibitem[{Baker(2013)}]{baker2013introduction}
Paul Baker. 2013.
\newblock Introduction: {V}irtual special issue of gender and language on
  corpus approaches.
\newblock \emph{Gender and Language}, 1(1).

\bibitem[{Baker(2014)}]{baker2014using}
Paul Baker. 2014.
\newblock \emph{Using corpora to analyze gender}.
\newblock A\&C Black.

\bibitem[{Bamman et~al.(2014)Bamman, Eisenstein, and Schnoebelen}]{Bamman2014}
David Bamman, Jacob Eisenstein, and Tyler Schnoebelen. 2014.
\newblock Gender identity and lexical variation in social media.
\newblock \emph{Journal of Sociolinguistics}, 18(2):135--160.

\bibitem[{de~Beauvoir(1953)}]{de1953second}
Simone de~Beauvoir. 1953.
\newblock \emph{The Second Sex}.
\newblock Vintage Books.

\bibitem[{Blei et~al.(2003)Blei, Ng, and Jordan}]{blei2003latent}
David~M. Blei, Andrew~Y. Ng, and Michael~I. Jordan. 2003.
\newblock Latent {D}irichlet allocation.
\newblock \emph{Journal of Machine Learning Research}, 3(Jan):993--1022.

\bibitem[{Bolukbasi et~al.(2016)Bolukbasi, Chang, Zou, Saligrama, and
  Kalai}]{bolukbasi2016man}
Tolga Bolukbasi, Kai-Wei Chang, James~Y. Zou, Venkatesh Saligrama, and Adam~T.
  Kalai. 2016.
\newblock Man is to computer programmer as woman is to homemaker? {D}ebiasing
  word embeddings.
\newblock In \emph{Advances in Neural Information Processing Systems}, pages
  4349--4357.

\bibitem[{{Caldas-Coulthard} and
  Moon(2010)}]{caldas-coulthardCurvyHunkyKinky2010b}
Carmen~Rosa {Caldas-Coulthard} and Rosamund Moon. 2010.
\newblock \href {https://doi.org/10.1177/0957926509353843} {`{{Curvy}}, hunky,
  kinky': {{Using}} corpora as tools for critical analysis}.
\newblock \emph{Discourse \& Society}, 21(2):99--133.

\bibitem[{Eisenstein et~al.(2011)Eisenstein, Ahmed, and
  Xing}]{eisensteinSparseAdditiveGenerative2011}
Jacob Eisenstein, Amr Ahmed, and Eric~P. Xing. 2011.
\newblock \href {http://dl.acm.org/citation.cfm?id=3104482.3104613} {Sparse
  additive generative models of text}.
\newblock In \emph{Proceedings of the 28th International Conference on
  International Conference on Machine Learning}, ICML'11, pages 1041--1048,
  USA. Omnipress.

\bibitem[{Ganchev et~al.(2010)Ganchev, Gillenwater, Taskar
  et~al.}]{ganchev2010posterior}
Kuzman Ganchev, Jennifer Gillenwater, Ben Taskar, et~al. 2010.
\newblock Posterior regularization for structured latent variable models.
\newblock \emph{Journal of Machine Learning Research}, 11(Jul):2001--2049.

\bibitem[{Garg et~al.(2018)Garg, Schiebinger, Jurafsky, and Zou}]{garg2018word}
Nikhil Garg, Londa Schiebinger, Dan Jurafsky, and James Zou. 2018.
\newblock Word embeddings quantify 100 years of gender and ethnic stereotypes.
\newblock \emph{Proceedings of the National Academy of Sciences},
  115(16):E3635--E3644.

\bibitem[{Goldberg and Orwant(2013)}]{goldberg-orwant:2013:*SEM}
Yoav Goldberg and Jon Orwant. 2013.
\newblock \href {http://www.aclweb.org/anthology/S13-1035} {A dataset of
  syntactic-ngrams over time from a very large corpus of {E}nglish books}.
\newblock In \emph{Second Joint Conference on Lexical and Computational
  Semantics (*SEM), Volume 1: Proceedings of the Main Conference and the Shared
  Task: Semantic Textual Similarity}, pages 241--247, Atlanta, Georgia, USA.
  Association for Computational Linguistics.

\bibitem[{Good(2004)}]{good2004permutation}
Phillip~I. Good. 2004.
\newblock \emph{Permutation, parametric, and bootstrap tests of hypotheses}.
\newblock Springer.

\bibitem[{Gordon and Van~Durme(2013)}]{gordonReportingBiasKnowledge2013}
Jonathan Gordon and Benjamin Van~Durme. 2013.
\newblock \href {https://doi.org/10.1145/2509558.2509563} {Reporting {{Bias}}
  and {{Knowledge Acquisition}}}.
\newblock In \emph{Proceedings of the 2013 {{Workshop}} on {{Automated
  Knowledge Base Construction}}}, {{AKBC}} '13, pages 25--30, New York, NY,
  USA. {ACM}.

\bibitem[{Hoyle et~al.(2019)Hoyle, Wolf-Sonkin, Wallach, Cotterell, and
  Augenstein}]{hoyleSentiment2019}
Alexander Hoyle, Lawrence Wolf-Sonkin, Hanna Wallach, Ryan Cotterell, and
  Isabelle Augenstein. 2019.
\newblock Combining disparate sentiment lexica with a multi-view variational
  autoencoder.
\newblock In \emph{Proceedings of the 2019 Conference of the North American
  Chapter of the Association for Computational Linguistics: Human Language
  Technologies, Volume 2 (Short Papers)}. Association for Computational
  Linguistics.

\bibitem[{Kingma and Ba(2015)}]{Kingma2014AdamAM}
Diederik~P. Kingma and Jimmy Ba. 2015.
\newblock Adam: {A} method for stochastic optimization.
\newblock In \emph{International Conference on Learning Representations
  (ICLR)}.

\bibitem[{Lakoff(1973)}]{lakoff1973language}
Robin Lakoff. 1973.
\newblock Language and woman's place.
\newblock \emph{Language in Society}, 2(1):45--79.

\bibitem[{McKee and Sherriffs(1957)}]{mckee1957differential}
John~P. McKee and Alex~C. Sherriffs. 1957.
\newblock The differential evaluation of males and females.
\newblock \emph{Journal of Personality}, 25(3):356--371.

\bibitem[{Miller et~al.(1993)Miller, Leacock, Tengi, and
  Bunker}]{miller1993semantic}
George~A. Miller, Claudia Leacock, Randee Tengi, and Ross~T. Bunker. 1993.
\newblock A semantic concordance.
\newblock In \emph{Proceedings of the workshop on Human Language Technology
  (HLT)}, pages 303--308. Association for Computational Linguistics.

\bibitem[{{Moon, Rosamund}(2014)}]{moonrosamundGorgeousGrumpyAdjectives2014}
{Moon, Rosamund}. 2014.
\newblock From gorgeous to grumpy: Adjectives, age, and gender.
\newblock \emph{Gender and Language}, 8(1):5--41.

\bibitem[{Norberg(2016)}]{norbergNaughtyBoysSexy2016a}
Cathrine Norberg. 2016.
\newblock \href {https://doi.org/10.1177/0075424216665672} {Naughty {{Boys}}
  and {{Sexy Girls}}: {{The Representation}} of {{Young Individuals}} in a
  {{Web}}-{{Based Corpus}} of {{English}}}.
\newblock \emph{Journal of English Linguistics}, 44(4):291--317.

\bibitem[{O'Neil(2016)}]{o2016weapons}
Cathy O'Neil. 2016.
\newblock \emph{Weapons of math destruction: {H}ow big data increases
  inequality and threatens democracy}.
\newblock Broadway Books.

\bibitem[{Ott(2016)}]{ott2016tweet}
Margaret Ott. 2016.
\newblock Tweet like a girl: {C}orpus analysis of gendered language in social
  media.
\newblock Bachelor's thesis, Yale University.

\bibitem[{Pearce(2008)}]{pearceInvestigatingCollocationalBehaviour2008}
Michael Pearce. 2008.
\newblock \href {https://doi.org/10.3366/E174950320800004X} {Investigating the
  collocational behaviour of man and woman in the {{BNC}} using {{Sketch
  Engine}}}.
\newblock \emph{Corpora}, 3(1):1--29.

\bibitem[{Rubin et~al.(1974)Rubin, Provenzano, and Luria}]{rubin1974eye}
Jeffrey~Z. Rubin, Frank~J. Provenzano, and Zella Luria. 1974.
\newblock The eye of the beholder: {P}arents' views on sex of newborns.
\newblock \emph{American Journal of Orthopsychiatry}, 44(4):512.

\bibitem[{Rudinger et~al.(2017)Rudinger, May, and
  Van~Durme}]{rudinger2017social}
Rachel Rudinger, Chandler May, and Benjamin Van~Durme. 2017.
\newblock Social bias in elicited natural language inferences.
\newblock In \emph{Proceedings of the First ACL Workshop on Ethics in Natural
  Language Processing}, pages 74--79.

\bibitem[{Rudinger et~al.(2018)Rudinger, White, and Van~Durme}]{N18-1067}
Rachel Rudinger, Aaron~Steven White, and Benjamin Van~Durme. 2018.
\newblock \href {https://doi.org/10.18653/v1/N18-1067} {Neural models of
  factuality}.
\newblock In \emph{Proceedings of the 2018 Conference of the North American
  Chapter of the Association for Computational Linguistics: Human Language
  Technologies, Volume 1 (Long Papers)}, pages 731--744. Association for
  Computational Linguistics.

\bibitem[{Schofield and Mehr(2016)}]{schofield-mehr:2016:CLfL2016}
Alexandra Schofield and Leo Mehr. 2016.
\newblock \href {http://www.aclweb.org/anthology/W16-0204}
  {Gender-distinguishing features in film dialogue}.
\newblock In \emph{Proceedings of the Fifth Workshop on Computational
  Linguistics for Literature}, pages 32--39, San Diego, California, USA.
  Association for Computational Linguistics.

\bibitem[{Storage et~al.(2016)Storage, Horne, Cimpian, and
  Leslie}]{storage2016frequency}
Daniel Storage, Zachary Horne, Andrei Cimpian, and Sarah-Jane Leslie. 2016.
\newblock The frequency of ``brilliant'' and ``genius'' in teaching evaluations
  predicts the representation of women and {A}frican {A}mericans across fields.
\newblock \emph{PloS one}, 11(3):e0150194.

\bibitem[{Tsvetkov et~al.(2014)Tsvetkov, Schneider, Hovy, Bhatia, Faruqui, and
  Dyer}]{TSVETKOV14.1096}
Yulia Tsvetkov, Nathan Schneider, Dirk Hovy, Archna Bhatia, Manaal Faruqui, and
  Chris Dyer. 2014.
\newblock Augmenting {E}nglish adjective senses with supersenses.
\newblock In \emph{Proceedings of the Ninth International Conference on
  Language Resources and Evaluation (LREC'14)}, Reykjavik, Iceland. European
  Language Resources Association (ELRA).

\bibitem[{Underwood et~al.(2018)Underwood, Bamman, and
  Lee}]{underwood2018transformation}
Ted Underwood, David Bamman, and Sabrina Lee. 2018.
\newblock \href {https://doi.org/10.22148/16.019} {The transformation of gender
  in english-language fiction}.
\newblock \emph{Journal of Cultural Analytics}.

\bibitem[{Williams and Bennett(1975)}]{williams1975definition}
John~E. Williams and Susan~M. Bennett. 1975.
\newblock The definition of sex stereotypes via the adjective check list.
\newblock \emph{Sex Roles}, 1(4):327--337.

\bibitem[{Williams and Best(1977)}]{williamsSexStereotypesTrait1977}
John~E. Williams and Deborah~L. Best. 1977.
\newblock \href {https://doi.org/10.1177/001316447703700111} {Sex
  {{Stereotypes}} and {{Trait Favorability}} on the {{Adjective Check List}}}.
\newblock \emph{Educational and Psychological Measurement}, 37(1):101--110.

\bibitem[{Williams and Best(1990)}]{williams+best}
John~E. Williams and Deborah~L. Best. 1990.
\newblock \emph{Measuring sex stereotypes: a multination study}.
\newblock Newbury Park, Calif. : Sage.

\bibitem[{Zhao et~al.(2017)Zhao, Wang, Yatskar, Ordonez, and
  Chang}]{zhao-EtAl:2017:EMNLP20173}
Jieyu Zhao, Tianlu Wang, Mark Yatskar, Vicente Ordonez, and Kai-Wei Chang.
  2017.
\newblock \href {https://www.aclweb.org/anthology/D17-1323} {Men also like
  shopping: Reducing gender bias amplification using corpus-level constraints}.
\newblock In \emph{Proceedings of the 2017 Conference on Empirical Methods in
  Natural Language Processing}, pages 2979--2989, Copenhagen, Denmark.
  Association for Computational Linguistics.

\bibitem[{Zhao et~al.(2018)Zhao, Wang, Yatskar, Ordonez, and Chang}]{N18-2003}
Jieyu Zhao, Tianlu Wang, Mark Yatskar, Vicente Ordonez, and Kai-Wei Chang.
  2018.
\newblock \href {https://doi.org/10.18653/v1/N18-2003} {Gender bias in
  coreference resolution: {E}valuation and debiasing methods}.
\newblock In \emph{Proceedings of the 2018 Conference of the North American
  Chapter of the Association for Computational Linguistics: Human Language
  Technologies, Volume 2 (Short Papers)}, pages 15--20. Association for
  Computational Linguistics.

\end{thebibliography}
\bibliographystyle{acl_natbib}

\clearpage
\appendix

\section{List of Gendered, Animate Nouns}\label{app:gendered-nouns}
\cref{tab:noun-list} contains the full list of gendered, animate nouns
that we use. We consider each row in this table to be the inflected
forms of a single lemma.

\begin{table}[ht]
  \centering
  \begin{adjustbox}{width=\columnwidth}
  \begin{tabular}{l l l l}
  \toprule
  \multicolumn{2}{c}{Male}&  \multicolumn{2}{c}{Female} \\
  Singular &  Plural & Singular & Plural \\
  \midrule
    man & men & woman & women\\
    boy & boys & girl & girls\\
    father & fathers & mother & mothers\\
    son & sons & daughter & daughters\\
    brother & brothers & sister & sisters\\
    husband & husbands & wife & wives\\
    uncle & uncles & aunt & aunts\\
    nephew & nephews & niece & nieces\\
    emperor & emperors & empress & empresses\\
    king & kings & queen & queens \\
    prince & princes & princess & princesses\\
    duke & dukes & duchess & duchesses\\
    lord & lords & lady & ladies\\
    knight & knights & dame & dames\\
    waiter & waiters & waitress & waitresses\\
    actor & actors & actress & actresses\\
    god & gods & goddess & goddesses\\
    policeman & policemen & policewoman & policewomen\\
    postman & postmen & postwoman & postwomen\\
    hero & heros & heroine & heroines\\
    wizard & wizards & witch & witches\\
    steward & stewards & stewardess & stewardesses\\
    he & -- & she & --\\
  \bottomrule
  \end{tabular}
  \end{adjustbox}
  \caption{Gendered, animate nouns.}\label{tab:noun-list}
\end{table}



\section{Relationship to PMI}\label{sec:pmi}
\setcounter{prop}{0}
\begin{prop}
Consider the following restricted version of our model. Let $\vf_{g}
\in \{0, 1\}^2$ be a one-hot vector that represents only the gender of
a noun. We write $g$ instead of $n$, equivalence-classing
all nouns as either \textsc{masc} or \textsc{fem}.  Let
$\veta^\star(\cdot) : \calV \rightarrow \mathbb{R}^2$ be the
maximum-likelihood estimate for the special case of our model without
(latent) sentiments:
\begin{equation}
    p(\nu \,|\, g) \propto \exp(m_\nu + \vf_g^{\top}\veta^\star(\nu)).
\end{equation}
Then, we have
\begin{align}
   \tau_g(\nu) \propto \exp(\text{PMI}(\nu, g)).
\end{align}
\end{prop}
\begin{proof}
First, we note our model has enough parameters to fit
the empirical distribution exactly:
\begin{align}
    \hat{p}(\nu \mid g) &= p(\nu \mid g) \\
                        &\propto \exp\{m_\nu + \vf_g^{\top} \veta^\star(\nu)\}.
\end{align}
Then, we proceed with an algebraic manipulation of the definition of
pointwise mutual information:
\begin{align}
    \textit{PMI}(\nu, g) &= \log \frac{\hat{p}(\nu, n)}{\hat{p}(\nu)\, \hat{p}(n)} \\
    &= \log \frac{\hat{p}(\nu \mid n)}{\hat{p}(\nu)} \\
    &= \log \frac{p(\nu \mid n)}{\hat{p}(\nu)} \\
    &= \log \frac{p(\nu \mid n)}{\exp\{m_\nu\}} \\
    &= \log \frac{1}{Z}\frac{\exp\{m_\nu + \vf_g^{\top} \veta^\star(\nu)\}}{\exp\{m_\nu\}} \\
     &= \log \frac{1}{Z}\exp\{\vf_g^{\top} \veta^\star(\nu)\} \\
      &= \vf_g^{\top} \veta^\star(\nu) - \log Z.
\end{align}
Now we have
\begin{align}
    \tau_g(\nu) &\propto \exp\{f_g^{\top} \veta^\star(\nu)\} \\
           &\propto \exp\{f_g^{\top} \veta^\star(\nu) - \log Z \}  \\
           &= \exp(\textit{PMI}(\nu, g)),
\end{align}
which is what we wanted to show.
\end{proof}

\section{Senses}\label{app:super-sense-categories}
In \cref{tab:super-sense-categories}, we list the senses for
adjectives \cite{TSVETKOV14.1096} and for
verbs~\cite{miller1993semantic}.
\begin{table}[ht]
  \centering
  \small
  \begin{tabular}{l l}
  \toprule
  Adjectives & Verbs \\
  \midrule
    Behavior & Body \\
    Body & Change \\
    Feeling & Cognition \\
    Mind & Communication \\
    Miscellaneous & Competition \\
    Motion & Consumption \\
    Perception & Contact \\
    Quantity & Creation \\
    Social & Emotion \\
    Spatial & Motion \\
    Substance & Perception \\
    Temporal & Possession \\
    Weather & Social \\
     & Stative \\
     & Weather \\
  \bottomrule
  \end{tabular}
  \caption{Senses for adjectives and verbs.}\label{tab:super-sense-categories}
\end{table}

\section{Additional Results}\label{sec:output}

In \cref{tab:nsubj-results} and \cref{tab:dobj-results}, we provide
the largest-deviation verbs used to describe male and female nouns for
\textsc{nsubj}--verb pairs and \textsc{dobj}--verb pairs.

\begin{table*}[ht]
  \centering
  \begin{adjustbox}{width=\textwidth}
  \begin{tabular}{lrlrlr|lrlrlr}
  \toprule
  \multicolumn{2}{c}{$\tau_\textsc{masc-pos}$}&  \multicolumn{2}{c}{$\tau_\textsc{masc-neg}$}& \multicolumn{2}{c|}{$\tau_\textsc{masc-neu}$}& \multicolumn{2}{c}{$\tau_\textsc{fem-pos}$}&  \multicolumn{2}{c}{$\tau_\textsc{fem-neg}$}& \multicolumn{2}{c}{$\tau_\textsc{fem-neu}$} \\
  Verb &  Value & Verb & Value & Verb & Value & Verb &  Value &
  Verb & Value & Verb & Value\\
  \midrule
       succeed &              1.6 &          fight &              1.2 &         extend &              0.7 &      celebrate &             2.4 &     persecute &             2.1 &         faint &             0.7 \\
       protect &              1.4 &           fail &              1.0 &          found &              0.8 &      fascinate &             0.8 &         faint &             1.0 &            be &             1.1 \\
         favor &              1.3 &           fear &              1.0 &         strike &              1.3 &     facilitate &             0.7 &           fly &             1.0 &            go &             0.4 \\
      flourish &              1.3 &         murder &              1.5 &            own &              1.1 &          marry &             1.8 &          weep &             2.3 &          find &             0.1 \\
       prosper &              1.7 &          shock &              1.6 &        collect &              1.1 &          smile &             1.8 &          harm &             2.2 &           fly &             0.4 \\
       support &              1.5 &          blind &              1.6 &            set &              0.8 &            fan &             0.8 &          wear &             2.0 &          fall &             0.1 \\
       promise &              1.5 &         forbid &              1.5 &            wag &              1.0 &           kiss &             1.8 &         mourn &             1.7 &          wear &             0.9 \\
       welcome &              1.5 &           kill &              1.3 &        present &              0.9 &       champion &             2.2 &          gasp &             1.1 &         leave &             0.7 \\
        favour &              1.2 &        protest &              1.3 &        pretend &              1.1 &          adore &             2.0 &       fatigue &             0.7 &          fell &             0.1 \\
         clear &              1.9 &          cheat &              1.3 &      prostrate &              1.1 &          dance &             1.7 &         scold &             1.8 &        vanish &             1.3 \\
        reward &              1.8 &           fake &              0.8 &           want &              0.9 &          laugh &             1.6 &        scream &             2.1 &          come &             0.7 \\
        appeal &              1.6 &        deprive &              1.5 &         create &              0.9 &           have &             1.4 &       confess &             1.7 &     fertilize &             0.6 \\
     encourage &              1.5 &       threaten &              1.3 &            pay &              1.1 &           play &             1.0 &           get &             0.5 &         flush &             0.5 \\
         allow &              1.5 &      frustrate &              0.9 &         prompt &              1.0 &           give &             0.8 &        gossip &             2.0 &          spin &             1.6 \\
       respect &              1.5 &         fright &              0.9 &         brazen &              1.0 &           like &             1.8 &         worry &             1.8 &         dress &             1.4 \\
       comfort &              1.4 &         temper &              1.4 &          tarry &              0.7 &         giggle &             1.4 &            be &             1.3 &          fill &             0.2 \\
         treat &              1.3 &        horrify &              1.4 &          front &              0.5 &          extol &             0.6 &          fail &             0.4 &           fee &             0.2 \\
         brave &              1.7 &        neglect &              1.4 &          flush &              0.3 &  compassionate &             1.9 &         fight &             0.4 &        extend &             0.1 \\
        rescue &              1.5 &          argue &              1.3 &          reach &              0.9 &           live &             1.4 &          fake &             0.3 &         sniff &             1.6 \\
           win &              1.5 &       denounce &              1.3 &         escape &              0.8 &           free &             0.9 &       overrun &             2.4 &     celebrate &             1.1 \\
          warm &              1.5 &        concern &              1.2 &             gi &              0.7 &     felicitate &             0.6 &          hurt &             1.8 &          clap &             1.1 \\
        praise &              1.4 &          expel &              1.7 &           rush &              0.6 &         mature &             2.2 &      complain &             1.7 &        appear &             0.9 \\
           fit &              1.4 &        dispute &              1.5 &      duplicate &              0.5 &          exalt &             1.7 &        lament &             1.5 &            gi &             0.8 \\
          wish &              1.4 &        obscure &              1.4 &      incarnate &              0.5 &        surpass &             1.7 &     fertilize &             0.5 &          have &             0.5 \\
         grant &              1.3 &           damn &              1.4 &         freeze &              0.5 &           meet &             1.1 &         feign &             0.5 &         front &             0.5 \\
  \bottomrule
  \end{tabular}
  \end{adjustbox}
  \caption{The largest-deviation verbs used to describe male and
    female nouns for \textsc{nsubj}--verb
    pairs.}\label{tab:nsubj-results}
\end{table*}

\begin{table*}[ht]
  \centering
  \begin{adjustbox}{width=\textwidth}
  \begin{tabular}{lrlrlr|lrlrlr}
  \toprule
  \multicolumn{2}{c}{$\tau_\textsc{masc-pos}$}&  \multicolumn{2}{c}{$\tau_\textsc{masc-neg}$}& \multicolumn{2}{c|}{$\tau_\textsc{masc-neu}$}& \multicolumn{2}{c}{$\tau_\textsc{fem-pos}$}&  \multicolumn{2}{c}{$\tau_\textsc{fem-neg}$}& \multicolumn{2}{c}{$\tau_\textsc{fem-neu}$} \\
  Verb &  Value & Verb & Value & Verb & Value & Verb &  Value & Verb &
  Value & Verb & Value\\
  \midrule
        praise &   1.7 &          fight &   1.8 &            set &   1.5 &         marry &  2.3 &        forbid &  1.3 &          have &  1.0 \\
         thank &   1.7 &          expel &   1.8 &            pay &   1.2 &        assure &  3.4 &         shame &  2.5 &        expose &  0.8 \\
       succeed &   1.7 &           fear &   1.6 &         escape &   0.4 &        escort &  1.2 &        escort &  1.3 &        escort &  1.4 \\
         exalt &   1.2 &         defeat &   2.4 &            use &   2.1 &       exclaim &  1.0 &       exploit &  0.9 &          pour &  2.1 \\
        reward &   1.8 &           fail &   1.3 &          expel &   0.9 &          play &  2.7 &          drag &  2.1 &         marry &  1.3 \\
       commend &   1.7 &          bribe &   1.8 &         summon &   1.7 &          pour &  2.6 &        suffer &  2.2 &          take &  1.1 \\
           fit &   1.4 &           kill &   1.6 &          speak &   1.3 &        create &  2.0 &         shock &  2.1 &        assure &  1.6 \\
       glorify &   2.0 &           deny &   1.5 &           shop &   2.6 &          have &  1.8 &        fright &  2.4 &     fertilize &  1.6 \\
         honor &   1.6 &         murder &   1.7 &  excommunicate &   1.3 &     fertilize &  1.8 &         steal &  2.0 &           ask &  1.0 \\
       welcome &   1.9 &         depose &   2.3 &         direct &   1.1 &           eye &  0.9 &        insult &  1.8 &       exclaim &  0.6 \\
        gentle &   1.8 &         summon &   2.0 &          await &   0.9 &           woo &  3.3 &     fertilize &  1.6 &         strut &  2.3 \\
       inspire &   1.7 &          order &   1.9 &          equal &   0.4 &         strut &  3.1 &       violate &  2.4 &          burn &  1.7 \\
        enrich &   1.7 &       denounce &   1.7 &        appoint &   1.7 &          kiss &  2.6 &         tease &  2.3 &          rear &  1.5 \\
        uphold &   1.5 &        deprive &   1.6 &        animate &   1.1 &       protect &  2.1 &       terrify &  2.1 &       feature &  0.9 \\
       appease &   1.5 &           mock &   1.6 &         follow &   0.7 &           win &  2.0 &     persecute &  2.1 &         visit &  1.3 \\
          join &   1.4 &        destroy &   1.5 &         depose &   1.8 &         excel &  1.6 &           cry &  1.8 &           saw &  1.3 \\
  congratulate &   1.3 &        deceive &   1.7 &           want &   1.1 &         treat &  2.3 &        expose &  1.3 &      exchange &  0.8 \\
         extol &   1.1 &           bore &   1.6 &          reach &   0.9 &          like &  2.2 &          burn &  2.6 &         shame &  1.6 \\
       respect &   1.7 &          bully &   1.5 &          found &   0.8 &     entertain &  2.0 &         scare &  2.0 &          fade &  1.2 \\
         brave &   1.7 &         enrage &   1.4 &         exempt &   0.4 &       espouse &  1.4 &      frighten &  1.8 &        signal &  1.2 \\
         greet &   1.6 &           shop &   2.7 &            tip &   1.8 &       feature &  1.2 &      distract &  2.3 &           see &  1.2 \\
       restore &   1.5 &          elect &   2.2 &          elect &   1.7 &          meet &  2.2 &          weep &  2.3 &       present &  1.0 \\
         clear &   1.5 &         compel &   2.1 &         unmake &   1.5 &          wish &  1.9 &        scream &  2.3 &         leave &  0.8 \\
        excite &   1.2 &         offend &   1.5 &          fight &   1.2 &        fondle &  1.9 &         drown &  2.1 &       espouse &  1.3 \\
       flatter &   0.9 &          scold &   1.4 &        prevent &   1.1 &           saw &  1.8 &          rape &  2.0 &          want &  1.1 \\
  \bottomrule
  \end{tabular}
  \end{adjustbox}
  \caption{The largest-deviation verbs used to describe male and
    female nouns for \textsc{dobj}--verb
    pairs.}\label{tab:dobj-results}
\end{table*}

\end{document}